\documentclass[jair,twoside,11pt,theapa,letterpaperx]{article}
\usepackage{jair,theapa,rawfonts}

\usepackage{graphicx}
\usepackage{amsmath,amsthm}

\newtheorem{definition}{Definition}
\newtheorem{proposition}{Proposition}

\newcommand{\attrib}[1]{\mbox{\sc\lowercase{#1}\ }}
\newcommand{\type}[1]{\mbox{\it #1\/}}

\begin{document}

\title{A Logic-based Approach for Recognizing Textual Entailment
  Supported by Ontological Background Knowledge}

\author{\name Andreas Wotzlaw \email wotzlaw@informatik.uni-koeln.de \\
        \addr Institut f\"ur Informatik, Universit\"at zu K\"oln\\
        Weyertal 121, 50931 K\"oln, Germany
       \AND
       \name Ravi Coote \email ravi.coote@fkie.fraunhofer.de \\
       \addr Fraunhofer-Institut FKIE\\
       Fraunhoferstr. 20, 53343 Wachtberg, Germany}

\maketitle

\begin{abstract}
  We present the architecture and the evaluation of a new system for
  recognizing textual entailment (RTE). In RTE we want to identify
  automatically the type of a logical relation between two input
  texts. In particular, we are interested in proving the existence of an
  entailment between them. We conceive our system as a modular
  environment allowing for a high-coverage syntactic and semantic text
  analysis combined with logical inference. For the syntactic and
  semantic analysis we combine a deep semantic analysis with a shallow
  one supported by statistical models in order to increase the quality
  and the accuracy of results. For RTE we use logical inference of
  first-order employing model-theoretic techniques and automated
  reasoning tools. The inference is supported with problem-relevant
  background knowledge extracted automatically and on demand from
  external sources like, e.g., WordNet, YAGO, and OpenCyc, or other,
  more experimental sources with, e.g., manually defined presupposition
  resolutions, or with axiomatized general and common sense
  knowledge. The results show that fine-grained and consistent knowledge
  coming from diverse sources is a necessary condition determining the
  correctness and traceability of results.
\end{abstract}

\section{Introduction}
\label{sec:introduction} 
In this paper\footnote{Preliminary versions of different parts of this
  paper appeared in~\citeA{wotzlaw2010towards}
  and~\citeA{wotzlaw2010recognizing}.} we present a new, logic-based
system for {\em recognizing textual entailment} (RTE). Our aim is to
provide it as a robust, modular, and highly adaptable environment for a
linguistically motivated large-scale semantic text analysis. In RTE
(see~\citeR{nle2009specialrte}, for a good introduction) we want to
identify automatically the type of a logical relation between two
written texts. In particular, we are interested in proving the existence
of an entailment between them. The concept of {\em textual entailment}
indicates the state in which the semantics of a natural language written
text can be inferred from the semantics of another one. RTE requires a
processing at the lexical, as well as at the semantic and discourse
levels with an access to vast amounts of problem-relevant background
knowledge. RTE is one of the greatest challenges for any natural
language processing (NLP) system. If it succeeds with reasonable
accuracy, it is a clear indication for some thorough understanding of
how language works~\cite{bos2005recognising}. As a generic problem it
has many useful applications in NLP~\cite{giampiccolo2007thirdpascal}.
Interestingly, many application settings like, e.g., information
retrieval, paraphrase acquisition, question answering, or machine
translation can fully or partly be modeled as
RTE~\cite{bentivogli2009fifthrte}. Entailment problems between natural
language texts have been studied extensively in the last few years,
either as independent applications or as a part of more complex systems
(e.g., RTE Challenges, see~\citeR{bentivogli2009fifthrte}).

We try to solve a given RTE problem by applying a {\em model-theoretic}
approach where a {\em formal semantic representation} of the RTE problem
is computed and used in the subsequent logical analysis. In our setting,
we try to recognize the type of the logical relation between two English
input texts, i.e., between the text $T$ (usually several sentences) and
the hypothesis $H$ (one short sentence). The following definition
specifies more formally the inference problems we consider in RTE
(see~\citeR{sandt1992presupposition}, for a similar definition).

\begin{definition}
\label{def:rte}
Let $T$ be a text consisting of several sentences and $H$ a hypothesis
expressed by a sentence. Given a pair $\{T,H\}$, find answers to the
following, mutually exclusive conjectures with respect to the
background knowledge relevant both for $T$ and $H$:
\begin{enumerate}
\item $T$ entails $H$,
\item $T \wedge H$ is inconsistent, i.e., $T \wedge H$ contains some
  contradiction, or
\item $H$ is informative with respect to $T$, i.e., $T$ does not entail  
  $H$ and $T \wedge H$ is consistent (contains no contradiction). This 
  holds, e.g., when $H$ and $T$ are completely unrelated.
\end{enumerate}
\end{definition}  
  
However, in contrast to {\em automated deduction} systems (e.g.,
presented in~\citeR{akhmatova2005textual}), which compare the atomic
propositions obtained from text $T$ and hypothesis $H$ in order to
determine the existence of entailment, we apply {\em logical inference
  of first-order}, similar to~\cite{bos2005recognising}. To compute
adequate semantic representations for input problems, we build on a
combination of deep and shallow techniques for semantic analysis. The
main problem with approaches processing the text in a shallow fashion is
that they can be tricked easily, e.g., by negation, or systematically
replacing quantifiers. Also an analysis solely relying on some deep
approach may be jeopardized by a lack of fault tolerance or robustness
when trying to formalize some erroneous text (e.g., with grammatical or
orthographical errors) or a shorthand note. The main advantage when
integrating deep and shallow NLP components is the increased robustness
of deep parsing by exploiting information for words that are not
contained in the deep lexicon. The type of unknown words can then be
guessed, e.g., by usage of statistical models.

The semantic representation language we use in our RTE system for the
results of the deep-shallow analysis is a first-order fragment of {\em
  Minimal Recursion Semantics} (MRS,
see~\citeR{Copestake2005a}). However, for their further usage in the
logical inference, the MRS expressions are translated into another,
semantic equivalent representation of {\em First-Order Logic with
  Equality} (FOLE). This logical form with a well-defined
model-theoretic semantics was successfully applied for RTE
by~\citeR{bos2005recognising}.

An adequate representation of a natural language semantics requires
access to vast amounts of common sense and domain-specific world
knowledge. Many applications in modern information technology utilize
ontological\footnote{In this paper we mean by ontology any set of
  facts and/or axioms comprising potentially both individuals (e.g.,
  London) and concepts (e.g., city).} knowledge to increase their
performance and to improve the quality of results. This applies in
particular to the applications from the Semantic Web, but also to
other domains. For instance machine translation exploits lexical
knowledge~\cite{chatterjee2005resolving}, document classification uses
ontologies~\cite{ifrim2006transductive}, whereas question
answering~\cite{hunt2004gazetteers}, information
retrieval~\cite{bentivogli2009fifthrte}, and textual
entailment~\cite{Bos2006a} rely strongly on background
knowledge. There are also emerging trends towards entity- and
fact-oriented Web search which can build on rich knowledge
bases~\cite{cafarella2007structured,mline2007knowledge}. Furthermore,
ontological knowledge structures play an important role in information
integration in general~\cite{noy2005semantic}.

Unfortunately, the existing applications today use typically only one
source of background knowledge, e.g.,
WordNet~\cite{fellbaum1998wordnet} or Wikipedia. They could boost
their performance if a huge ontology with knowledge from several
sources was available. Such a knowledge base would have to be of high
quality and accuracy comparable with that of an encyclopedia. It
should include not only ontological concepts and lexical hierarchies
like those of WordNet, but also a great number of named entities (here
also referred to as individuals) like, e.g., people, geographical
locations, organizations, events, etc. Also other semantic relations
between them, e.g., who-was-born-when, which-language-is-spoken-in,
etc. should be comprised.

RTE systems need problem-relevant background knowledge to support
their proofs. The logical inference in our system is supported by
external background knowledge integrated automatically and only as
needed into the input problem in form of additional first-order
axioms. In contrast to already existing applications (see,
e.g.,~\citeR{curran2007lingu,bentivogli2009fifthrte}), our system
enables flexible integration of background knowledge from more than
one external source (see Section~\ref{ssec:knowledge} for details). In
its current implementation, our system supports RTE, but can also be
used for other NLP tasks like, e.g., large-scale syntactic and
semantic analysis of English texts, or multilingual information
extraction.

The rest of the paper is organized as follows. In
Section~\ref{sec:related}, we review some related works and compare
our method with the existing ones. In Section~\ref{sec:framework}, we
introduce the architecture of our system for RTE.  In
Section~\ref{sec:quality}, we show how the quality of the logic-based
inference process of our system can further be improved. Finally in
Section~\ref{sec:conclusions}, we discuss the evaluation results and
conclude the paper.

\section{Related Work.} 
\label{sec:related}
As a generic problem RTE has many applications in NLP which have been
studied extensively in the last few years. We refer the reader
to~\citeA{nle2009specialrte},~\citeA{bentivogli2009fifthrte},
and~\citeA{androutsopoulos2010survey} for good overviews. Our work was
mostly inspired by the ideas given by~\citeA{Blackburn2005a}, and
\citeA{bos2005recognising}, where a similar model-theoretic approach was
used for the semantic text analysis with logical inference. However, in
contrast to our logic-based approach, they applied another, more
discourse-oriented semantic formalism, {\em Discourse Representation
  Theory}~\cite{KampReyle1993}, for the computation of full semantic
representations. Furthermore, in our system the framework {\em Heart of
  Gold} by~\citeA{Schaefer2007a} was used as a basis for the semantic
analysis. For a good overview on a combined application of deep and
shallow NLP methods for RTE, we refer to~\citeA{Schaefer2007a},
and~\citeA{Bos2005a}. The application of logical inference techniques
for RTE was already elaborately presented
in~\citeA{blackburn1998automated,bos2005recognising},
and~\citeA{Bos2006a}. Moreover, an interesting discussion on the
importance of WordNet as a source of background knowledge for RTE can be
found in~\citeA{bos2005towards},
and~\citeA{Bos2006a}. \citeA{tatu2006logicbased} proposed for RTE a new
knowledge representation model and a logic proving setting with axioms
on demand. A detailed discussion on formal methods for the analysis of
the meaning of natural language expressions can be found
in~\citeA{bos2008argue}. Furthermore, a new approach for RTE using {\em
  natural logics} was given by~\citeA{maccartney2009extended}. The
authors propose a promising, annotation-based model of natural language
inference which identifies valid inferences by their lexical and
syntactic features, without full semantic interpretation. Note that in
the system presented in this paper we aim at computing a full semantic
representation of the input texts.

Finally, the inference process of our system applies ontological
knowledge coming from different external sources like, e.g., WordNet
or YAGO~\cite{suchanek2008yago}. In the last few years we have been
observing a reasonable growth of interest in huge ontologies and their
applications. There exists also a number of huge ontology integration
projects like, e.g., YAGO together with the Suggested Upper Model
Ontology (SUMO, see~\citeR{melo2008integrating}),
DBpedia~\cite{auer2007dbpedia}, or the Linking Open Data
Project~\cite{bizer2008linked} which aim is to extract and to combine
ontological data from many sources. Since YAGO is a part of those
ontology projects, it should be possible to integrate them (at least
partly) into the RTE application by applying the integration procedure
from Section~\ref{sec:quality}.

\section{System Architecture}
\label{sec:framework}
Our system for RTE provides the user with a number of essential
functionalities for syntactic, semantic, and logical textual analysis
which can selectively be overridden or specialized in order to provide
new or more specific ones, e.g., for anaphora resolution or word sense
disambiguation. In its initial form, the application supplies, among
other things, flexible software interfaces and transformation
components, allows for execution of a deep-shallow syntactic and
semantic analysis, integrates external inference machines and
background knowledge, maintains and controls the semantic analysis and
the inference process, and provides the user with a graphical
interface for control and presentation purposes.

In the following we describe our system for RTE in more detail. It
consists of three main modules (see Figure~\ref{fig:framework}):
\begin{enumerate}
\item {\em Syntactic and Semantic Analysis}, where the combined
  deep-shallow semantic analysis of the input texts is performed;
\item {\em Logical Inference}, where the logical deduction process is
  implemented (it is supported by two external components integrating
  external knowledge and inference machines);
\item {\em Graphical User Interface}, where the analytical process is
  supervised and its results are presented to the user.
\end{enumerate}
\begin{figure}[h]
\begin{center}
\includegraphics[width=10cm]{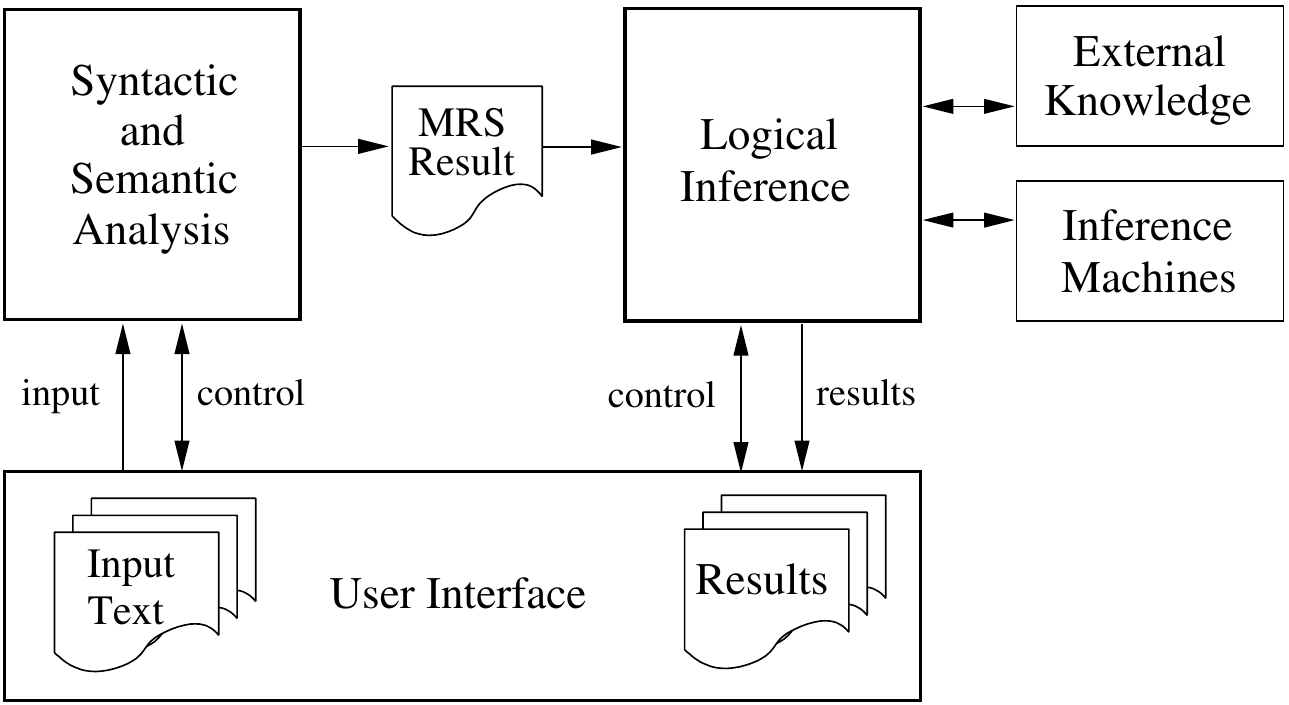}
\caption{Overall architecture of the system}
\label{fig:framework}
\end{center}
\end{figure}

In the following we describe the way the particular modules of the
system work. To make our description as comprehensible as possible, we
make use of a small RTE problem. More specifically, with its help we
show some crucial aspects of how our system proceeds while solving
some RTE problem.  We want to identify the logical relation between
text $T$:
\begin{center}
  \parbox{12cm}{\em Tower Bridge is one of the most recognizable
    bridges in the world. Many falcons inhabit its old roof nowadays.}
\end{center}
and hypothesis $H$:
\begin{center}
\parbox{12cm}{\em Birds live in London.}  
\end{center}
To prove the entailment automatically, among other things, a precise
semantic representation of the problem must be computed, the anaphoric
reference in $T$ must be resolved, and world knowledge (e.g., {\em
  Tower Bridge} is in {\em London}) and ontological relations between
the concepts (e.g., that {\em falcons} are {\em birds}) must be
provided to the logical inference. In the following we show how our
system proceeds while computing the first-order semantic
representation of the input texts above.

\subsection{Syntactic and Semantic Analysis} 
The texts of the input RTE problem after entering the system via the
user interface (see Figure~\ref{fig:framework}) go first through the
syntactic processing and semantic construction of the first system
module. To this end, they are analyzed by components of an XML-based
middleware architecture {\em Heart of Gold} (see
Figure~\ref{fig:syntax}). It allows for a flexible integration of
various shallow and deep linguistics-based and semantics-oriented NLP
components, and thus constitutes a sufficiently complex research
instrument for experimenting with novel processing strategies. Here,
we use its slightly modified standard configuration for English
centered around the English Resource HPSG Grammar (ERG,
see~\citeR{flickinger2000building}). The shallow processing is
performed through statistical or simple rule-based, typically
finite-state methods, possessing sufficient precision and recall. The
particular tasks are as follows: tokenization with the Java tool JTok,
part-of-speech tagging with a statistical tagger TnT
of~\citeA{brants2000tnt} trained for English on the Penn
Treebank~\cite{marcus1993building}, and named entity recognition with
SProUT~\cite{drozdz2004shallow}. The last one, by combining finite
state and typed feature structure technology, plays an important role
for the deep-shallow integration, i.e., it prepares the generic named
entity lexical entries for the deep HPSG parser PET
of~\citeA{callmeier2000pet}. This makes sharing of linguistic
knowledge among deep and shallow grammars natural and easy.
\begin{figure}
\begin{center}
\includegraphics[width=7cm]{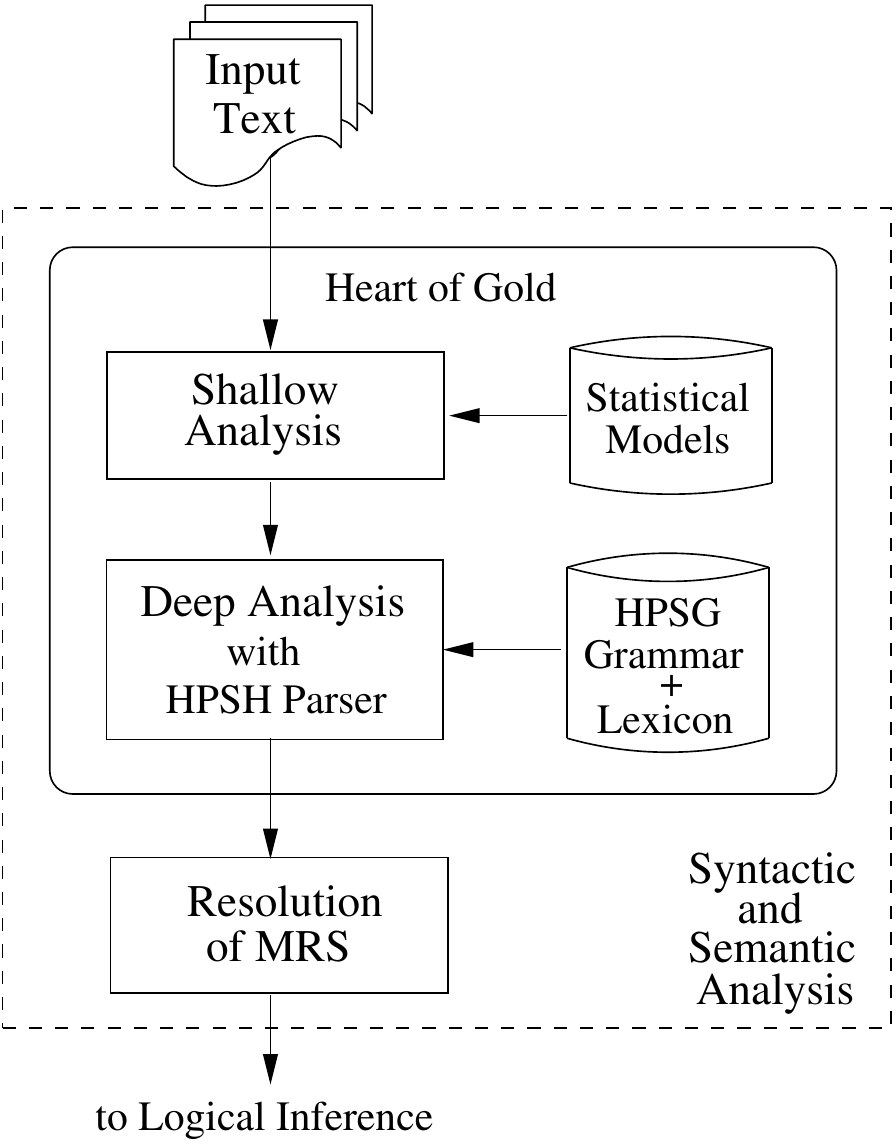}
\caption{Syntactic and semantic analysis}
\label{fig:syntax}
\end{center}
\end{figure}
PET is a highly efficient runtime parser for unification-based
grammars (like, e.g., ERG) and constitutes the core of the rule-based,
fine-grained deep semantic analysis. The integration of NLP components
is done either by means of an XSLT-based transformation, or with the
help of {\em Robust Minimal Recursion Semantics} (RMRS,
see~\citeR{copestake2003rmrs}), provided a given NLP component
supports it natively. RMRS is a generalization of MRS. It can not only
be underspecified for scope as MRS, but also partially specified,
e.g., when some parts of the text cannot be resolved by a given NLP
component. Thus, RMRS is well suited for representing output also from
shallow NLP components. This can be seen as a clear advantage over
approaches based strictly on some specified semantic representation
like those presented in, e.g.,~\citeA{blackburn1998automated},
or~\citeA{bos2005recognising}.

Furthermore, RMRS is a common semantic formalism for HPSG grammars
within the context of the {\em LinGO Grammar
  Matrix}~\cite{bender2002grammar}. Besides ERG, which we use for
English, there are also grammars for other languages like, e.g., the
Japanese HPSG grammar {\em JaCY}~\cite{siegel2002efficient}, the {\em
  Korean Resource Grammar}~\cite{jongbok2005parsing}, the {\em Spanish
  Resource Grammar} (SRG, see~\citeR{marimon2002integrating}), or the
proprietary German grammar~\cite{cramer2009construction}. Since all of
those grammars can be used to generate semantic representations in
form of RMRS, a replacement of ERG with another grammar in our system
can be considered and thus a higher degree of multilinguality
achieved. 

The combined results of the deep-shallow analysis in the RMRS form are
transformed into MRS and resolved with Utool 3.1
of~\citeA{koller2005efficient}. Utool translates first the input from
MRS into dominance constraints~\cite{thater2007minimal}, a closely
related scope underspecification formalism, and then enumerates all
readings represented by the dominance graph. In the current
implementation one of the most reasonable readings is chosen manually
by the analyst for the further processing in the logical
inference. For more detail on the production of first-order
expressions from a broad coverage HPSG grammar ERG we refer
to~\citeA{coote2010generation}.

For our small RTE example from the beginning of the section, the
result of the combined syntactic and semantic analysis for $H$ in form
of RMRS, given as attribute value matrix, is presented in
Figure~\ref{fig:rmrs}.
\begin{figure}
{\footnotesize
\[\left[\begin{array}{ll}
\attrib{TEXT}&\mbox{}\\
\attrib{TOP}&h1\\
\attrib{RELS}&\left\{\begin{array}{l}
	  
\left[\begin{array}{ll}
\multicolumn{2}{l}{\type{{\bf locname\_rel}}}\\
\attrib{LBL}&
\mbox{\it h1900108\/}
        \\
\attrib{ARG0}
&
\mbox{\it x1900108\/}
        \\
\attrib{CARG}&
\mbox{london}
\mbox{\it \/}
        \\
\attrib{ARG1}&
\mbox{\it x9\/}
        \\
            \end{array}\right]
\left[\begin{array}{ll}
\multicolumn{2}{l}{\type{{\bf loctype\_rel}}}\\
\attrib{LBL}&
\mbox{\it h1900110\/}
        \\
\attrib{ARG0}
&
\mbox{\it x1900110\/}
        \\
\attrib{CARG}&
\mbox{city}
\mbox{\it \/}
        \\
\attrib{ARG1}&
\mbox{\it x9\/}
        \\
            \end{array}\right]

\left[\begin{array}{ll}
\multicolumn{2}{l}{\type{udef\_q\_rel}}\\
\attrib{LBL}&
\mbox{\it h3\/}
        \\
\attrib{ARG0}
&
\mbox{\it x6\/}
{\tiny\begin{array}{l}\mbox{
     
}\end{array}}
        \\
\attrib{RSTR}&
\mbox{\it h5\/}
        \\
\attrib{BODY}&
\mbox{\it h4\/}
        \\
            \end{array}\right]
\left[\begin{array}{ll}
\multicolumn{2}{l}{\type{\_bird\_n}}\\
\attrib{LBL}&
\mbox{\it h7\/}
        \\
\attrib{ARG0}
&
\mbox{\it x6\/}
{\tiny\begin{array}{l}\mbox{
}\end{array}}
        \\
            \end{array}\right]
\\
\left[\begin{array}{ll}
\multicolumn{2}{l}{\type{\_live\_v}}\\
\attrib{LBL}&
\mbox{\it h8\/}
        \\
\attrib{ARG0}
&
\mbox{\it e2\/}
{\tiny\begin{array}{l}\mbox{
}\end{array}}
        \\
\attrib{ARG1}&
                  
\mbox{\it x6\/}
{\tiny\begin{array}{l}\mbox{
      
}\end{array}}
        \\
            \end{array}\right]

\left[\begin{array}{ll}
\multicolumn{2}{l}{\type{\_in\_p}}\\
\attrib{LBL}&
\mbox{\it h10001\/}
        \\
\attrib{ARG0}
&
\mbox{\it e10\/}
{\tiny\begin{array}{l}\mbox{
  
}\end{array}}
        \\
            
\attrib{ARG1}&
                  
\mbox{\it e2\/}
{\tiny\begin{array}{l}\mbox{
     
}\end{array}}
        \\
            
\attrib{ARG2}&
                  
\mbox{\it x9\/}
{\tiny\begin{array}{l}\mbox{
     
}\end{array}}
        \\
            \end{array}\right]
	  
\left[\begin{array}{ll}
\multicolumn{2}{l}{\type{proper\_q\_rel}}\\
\attrib{LBL}&
\mbox{\it h11\/}
        \\
\attrib{ARG0}
&
\mbox{\it x9\/}
{\tiny\begin{array}{l}\mbox{
  
}\end{array}}
        \\
            
\attrib{RSTR}&
                  
\mbox{\it h13\/}
        \\
            
\attrib{BODY}&
                  
\mbox{\it h12\/}
        \\
            \end{array}\right]
\left[\begin{array}{ll}
\multicolumn{2}{l}{\type{named\_rel}}\\
\attrib{LBL}&
\mbox{\it h14\/}
        \\
\attrib{ARG0}
&
\mbox{\it x9\/}
{\tiny\begin{array}{l}\mbox{
  
}\end{array}}
        \\
            
\attrib{CARG}&
                  
\mbox{London}
\mbox{\it \/}
        \\
            \end{array}\right]
	  \end{array}\right\}\\
  
\attrib{HCONS}&\{
        
\mbox{\it h5\/}
        
          \mbox{ qeq }
\mbox{\it h7\/}
        , 
\mbox{\it h13\/}
          \mbox{ qeq }
\mbox{\it h14\/}
        \}\\
\attrib{ING}&\{
\mbox{\it h8\/}
        \mbox{ ing }
\mbox{\it h10001\/}
        \}\\
\end{array}\right]\]
}
\caption{RMRS as attribute value matrix for hypothesis $H$ from the 
example}
\label{fig:rmrs}
\end{figure}
The results from the shallow analysis (marked bold) describe the named
entities from $H$. As described above, in the next step the structure
is transformed into MRS and resolved by Utool. The resulting
first-order MRS for the hypothesis $H$ from our example is given below
(in Prolog notation). The predicates with {\tt \_q\_}, {\tt \_n\_},
{\tt \_v\_}, and {\tt \_p\_} in their names represent quantifiers,
nouns, verbs, and prepositions, respectively.  {
\begin{verbatim}
udef_q_rel(X6, 
    bird_n_1_rel(X6), 
    proper_q_rel( X9, and(
        named_rel(X9, london), and(
        locname_rel(london, X9), 
        loctype_rel(city, X9))), and(
        live_v_1_rel(E2, X6), 
        in_p_dir_rel(E10, E2, X9)))).
\end{verbatim}
}

\subsection{Logical Inference} 
The results of the semantic analysis in form of specified MRS combining
deep-shallow predicates are translated into another, logical equivalent
semantic representation FOLE (see Figure~\ref{fig:semantics}). The
rule-based transformation conveys argument structure with a
neo-Davidsonian analysis with semantic roles~\cite{Dowty1989}. A
definite article is translated according to the theory of definite
description of~\citeA{Russell1905}. Temporal relations are modeled by
adding additional predicates similar to~\citeA{bos2005recognising},
and~\citeA{curran2007lingu}, i.e., without explicit usage of time
operators. Furthermore, it is possible to extend the translation
mechanism to cover plural and modal forms. Appropriate ideas can be
found in~\citeA{curran2007lingu} and~\citeA{Lohnstein1996}. In this
case, however, the complexity and amount of the resulting FOLE formulas
will grow rapidly, making the problem much harder to solve with the
currently available inference machines (see
Section~\ref{ssec:inference}).
\begin{figure}
\begin{center}
\includegraphics[width=14.4cm]{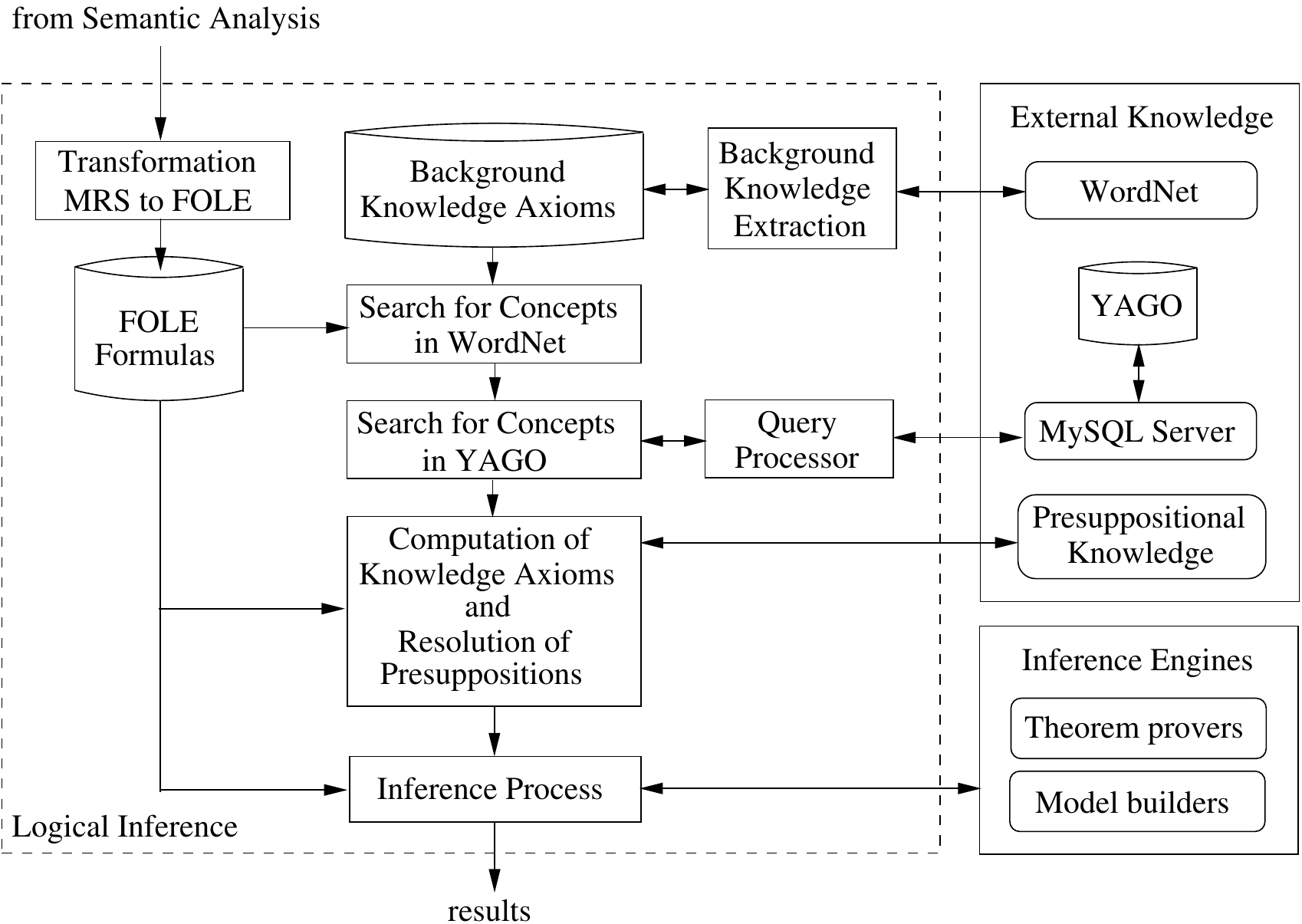}
\caption{Logical inference with external inference machines and
  background knowledge}
\label{fig:semantics}
\end{center}
\end{figure}

The translated FOLE formulas are stored locally and can be used for
the further analysis. Furthermore, such formally expressed input text
can and {\em should} be extended with additional knowledge in form of
{\em background knowledge axioms}. In our system, the additional
axioms are formulated in FOLE and integrated into the input
problem. The integration of background knowledge will be discussed in
detail in Section~\ref{sec:quality}.

As an example here, the translation of the specified MRS into FOLE for
the hypothesis $H$ from our example given earlier produces the
following formula with a neo-Davidsonian event representation:

{
\begin{verbatim}
some(X6,and(
    bird_n_1(X6),
    some(X9,and(and(
        named_r_1(X9),and(
        location_n_1(X9),and(
        london_loc_1(X9),
        city_n_1(X9)))),
        some(E2,and(
            event_n_1(E2),and(and(
            live_v_1(E2),
            agent_r_1(E2,X6)),
            in_r_1(E2,X9)))))))).
\end{verbatim}
}

\subsection{Inference Process} 
\label{ssec:inference}
The goal here is to identify and to prove the logical relation between
two input texts, represented formally as FOLE formulas, with respect to
the problem-relevant background knowledge. We are interested in
answering the question whether the relation is an entailment, a
contradiction, or whether maybe hypothesis $H$ provides just new
information with respect to text $T$, i.e., is informative (see
Definition~\ref{def:rte}). As proposed by~\citeR{bos2005recognising}, in
order to give a clear answer to these questions, it is sufficient to
perform sequentially the following three tests on text $T$, hypothesis
$H$, and background knowledge $BK$, all given as FOLE formulas:
\begin{enumerate}
\item Consistency Test: Check whether $T$, $H$, and $BK$ are mutually
  consistent in terms of first-order logic, i.e., if $$T\wedge H\wedge
  BK$$ is satisfiable (i.e., there exists some first-order model for
  it).
\item Informativity Test: Check whether $H$ contains new information
  (first-order assertions) which cannot be entailed from $T$ and $BK$,
  i.e., if $$T\wedge BK \rightarrow H$$ is not valid, or put another
  way, if $$\neg(T\wedge BK \rightarrow H)\equiv T\wedge BK \wedge
  \neg H$$ is satisfiable.
\item Entailment Test: Check whether $H$ is a semantic consequence of
  $T$ and $BK$, i.e., if $$\{T,BK\}\models H$$ holds. Since in the
  first-order logic, $\{T,BK\}\models H$ holds if and only if $T\wedge
  BK \rightarrow H$ is valid and $T\wedge BK \rightarrow \neg H$ is
  not valid (according to the semantic version of the deduction
  theorem, see, e.g.,~\citeR{boolos1974computability}), it suffices to
  show that $$T\wedge BK \wedge \neg H \mbox{ and } T\wedge BK \wedge
  H$$ are unsatisfiable and satisfiable, respectively.
\end{enumerate}

Observe that the three logical relations between $T$, $H$, and $BK$ we
are considering here are mutually exclusive and partly
complementary. Therefore the testing can be reduced to the first two
tests, i.e, we need to perform only the consistency and the
informativity tests. Figure~\ref{fig:diagram} shows how a given RTE
problem is solved by applying only these two tests.
\begin{figure}
\begin{center}
\includegraphics[width=13cm]{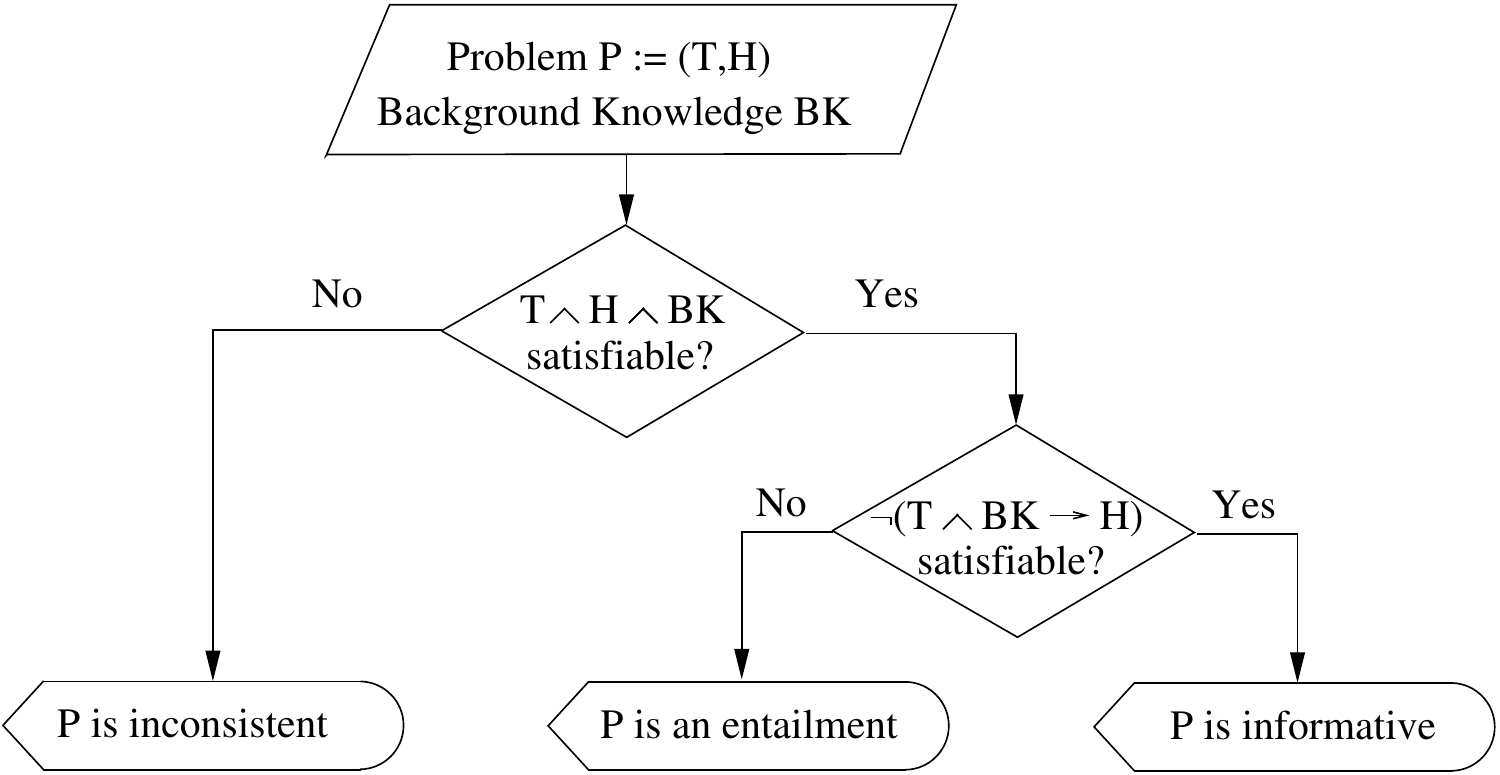}
\caption{Decision diagram for logical inference}
\label{fig:diagram}
\end{center}
\end{figure}

It is a well-know fact that our tests are within the first-order logic
{\em undecidable}. Thus, in order to check efficiently which type of a
logical relation for the input problem holds, we use two kinds of
automated reasoning tools:
\begin{itemize}
\item {\em Finite model builders}: Mace 2.2
  by~\citeA{mccune2001mace2}, Paradox 3.0 by~\citeA{claessen2003new},
  and Mace4 by~\citeA{mccune2003mace4}, and
\item {\em First-order provers}: Bliksem 1.12
  by~\citeA{nivelle2003bliksem}, Otter 3.3 by~\citeA{mccune2003otter},
  Vampire 8.1 by~\citeA{riazanow2002design}, and Prover9
  by~\citeA{mccune2009prover}.
\end{itemize}
While theorem provers are designed to prove that a formula is valid
(i.e., the formula is true for any admissible interpretation), they are
generally not good at deciding that a formula is not valid. On the
contrary, model builders are designed to show that a formula is true for
at least one interpretation. The experiments with different inference
machines show that solely relying on theorem proving is in most cases
insufficient due to low recall. Indeed, our inference process
incorporates model building as a central part of the inference process.
Similar to~\citeA{bos2003exploring}, \citeA{bos2005recognising},  
and~\citeA{curran2007lingu}, we exploit the complementarity of model
builders and theorem provers by applying them {\em in parallel} to the
first-order formulas specified by the first two tests given above in
order to tackle with its undecidability more efficiently. More
specifically, for a given test, the theorem prover attempts to prove
whether the input formula $F$ is valid whereas the model builder
simultaneously tries to find a finite, first-order model for the
negation of the input formula $F$. For more clarity, in the decision
nodes in Figure~\ref{fig:diagram}, the input formula is depicted only in
the form the model builder becomes it, i.e., in negated form as
satisfiability problem.

All reasoning machines were developed to deal with inference problems
stated in FOLE. They are successfully integrated into our system for
RTE. To this end, we use a translation from FOLE into the formats
required by the inference tools. Furthermore, the user can specify via
the user interface which inference machines (i.e., which theorem
prover and which model builder) should be used by the inference
process. The tests have shown that the efficiency and the success of
solving a given RTE problem depend much on the inference machines
chosen for it. Thus, it is advisable to run simultaneously on the same
RTE problem more than one theorem prover and more than one model
builder.

\subsection{User Interface} 
The results of the syntactic processing, semantic construction, and
logical inference like, e.g., HPSG and MRS structures, FOLE formulas,
first-order models and proofs, integrated background knowledge, and
other detailed information are presented to the user within a
dedicated GUI. With its help, one can further customize and control
both the semantic and logical analysis, e.g., choose the input text or
the background knowledge source, inspect the results of shallow-deep
analysis, or select and customize inference machines.

\section{Improving the Inference Quality}
\label{sec:quality}
The inference process of RTE needs high-quality background knowledge
to support its proofs. In particular, this will improve the precision
and the success rate of the inference process making the result much
more conforming with real-world expectations. However, with increasing
number of background knowledge axioms the search for finite
first-order models may become more time-consuming. Thus, only
knowledge relevant for the problem should be considered in the
inference process.

The integration of many ontological sources is, in general, a
difficult but as argued before a very important task. First of all,
the semantics of all concepts, individuals, and relations must be
preserved across the various sources. In this section we present and
analyze formally a new graph-based technique for integration of
concepts and individuals from ontologies based on the hierarchy of
WordNet~\cite{fellbaum1998wordnet}. Our results show that a
fine-grained and consistent knowledge coming from diverse sources (and
domains) is a necessary condition determining the correctness and
traceability of results. Moreover, our RTE application performs
significantly better when a substantial amount of problem-relevant
knowledge has been integrated into the underlying inference process.

\subsection{Sources of Background Knowledge}
\label{ssec:knowledge}
Our RTE system supports the extraction of background knowledge from
different kinds of sources. It searches for and supplies
problem-relevant knowledge automatically as first-order axioms and
integrates them into the RTE problem.

We use WordNet 3.0 as a lexical database for synonymy, hyperonymy, and
hyponymy relations. It helps the RTE system to detect an entailment
between lexical units from the text and the hypothesis. It serves also
as a database for individuals but rather a very small one when
compared to the second source. For efficiency purposes, it was
preprocessed and integrated directly into the module for logical
inference (see Figure~\ref{fig:semantics}). Conceptually, the
hyperonymy/hyponymy relation in WordNet spans a directed acyclic graph
(DAG) with the root node {\tt
  entity}~\cite{fellbaum1998wordnet,suchanek2008yago}. This means that
there are nodes (i.e., concepts or individuals) in the WordNet graph
that are direct hyponyms of more than one concept. For that reason the
knowledge axioms which are generated later from the WordNet graph may
induce inconsistencies between the input problem formulas and the
extracted knowledge. This can be very harmful for the further
inference process. In Section~\ref{ssec:integration} we discuss this
problem more formally and present several strategies that can deal
with this restriction.

YAGO~\cite{suchanek2008yago} is a large and arbitrarily extensible
ontology with high precision and quality which we use in our system as
the second source of ontological knowledge. Its core was assembled
automatically from the category system and the infoboxes of Wikipedia,
and combined with taxonomic relations from
WordNet~\cite{suchanek2008yago}. Similar to WordNet, the concepts and
individuals hierarchy of YAGO spans a DAG. Thus, we must also proceed
carefully when integrating data from that source into the RTE problem,
(see Section~\ref{ssec:integration}). For accessing YAGO, we use a
dedicated query processor (see Figure~\ref{fig:semantics}) with its
own query language, similar to that of~\citeA{suchanek2008yago}. The
query processor first normalizes the shorthand notation of the query,
and after translating it into SQL, sends it to the MySQL-Server with
YAGO database. The incoming results are first preprocessed by the
query processor, so that only those concepts are sent back for
integration which are consistent with WordNet concept hierarchy, i.e.,
which include the prefix {\tt wordnet\_}.

Furthermore, OpenCyc 2.0~\cite{matuszek2006introduction} can also be
used as a background knowledge source. The computation of axioms for a
given problem is solved using a variant of Lesk's WSD
algorithm~\cite{banerjee2002adapted}. Axioms of generic knowledge from
our experimental knowledge source cover the semantics of possessives,
active-passive alternation, and spatial knowledge (e.g., that Tower
Bridge is located in London). Finally, our experimental
presuppositional knowledge base includes axioms covering English words
and phrases triggering presuppositions (see
Section~\ref{ssec:presuppositions}).

\subsection{Combining Knowledge from Various Sources}
\label{ssec:integration}
In the following we describe the {\em three-phase} integration
procedure that we use to find and to combine individuals and concepts
from YAGO with those from WordNet in order to support RTE. In
particular, we show how we can combine problem-relevant individuals
and concepts from YAGO with those from WordNet so that the consistency
of background knowledge axioms is preserved whereas the original
logical properties of the input RTE problem do not change. More
specifically, Since the input problem itself may be consistent and we
want to prove it, the knowledge we integrate into it must not make it
inconsistent.

To make our presentation as comprehensible and self-explanatory as
possible, we make use of a small RTE problem which we augment with
relevant background knowledge axioms in the course of this section. We
want to prove that the text $T$:
\begin{center} 
  \parbox{12cm}{\em Leibniz was a famous German philosopher and
    mathematician born in Leipzig. Thomas reads his philosophical
    works while waiting for a train at the station of Bautzen.}
\end{center}
entails the hypothesis $H$: 
\begin{center}
\parbox{12cm}{\em Some works of Leibniz are read in a town.}
\end{center}
In order to prove the entailment above, we must know, among other
things, that {\em Bautzen} is a town. We assume that no information
about {\em Bautzen}, except that it is a named entity (i.e., an
individual), were yielded by the deep-shallow semantic
analysis. However, we expect that this missing information can be
found in the external knowledge sources.  The search for relevant
background knowledge begins after the first-order representation of
the problem is computed and translated into FOLE (see
Section~\ref{sec:framework}). At this stage, the RTE problem has
already undergone syntactic processing, semantic construction, and
anaphora resolution which together have generated a set of semantic
representations of the problem in form of MRS. The translation of the
specified MRS into FOLE for the hypothesis $H$ from our example above
produces the following formula with a neo-Davidsonian event
representation by~\citeA{Dowty1989}:

\begin{verbatim}
some(X3,and(
    work_n_2(X3),
    some(X7,and(and(
        named_r_1(X7),and(
        leibniz_per_1(X7),
        of_r_1(X3,X7))),
        some(X8,and(
            town_n_1(X8),
            some(E2,and(
                event_n_1(E2),and(and(
                read_v_1(E2),
                patient_r_1(E2,X3)),
                in_r_1(E2,X8)))))))))).
\end{verbatim}

As mentioned before, the integration procedure is composed of three
phases. In the first phase we search for problem-relevant knowledge in
WordNet, whereas in the second phase we look for additional knowledge
in YAGO which we combine afterwards with that found in the first
phase. Finally, in the third phase we generate from the knowledge, we
have already found and successfully combined, background knowledge
axioms and integrate them into the set of FOLE formulas representing
the input RTE problem.

\subsubsection{Phase I: Integration of WordNet}
\label{sssec:firstphase}
At the beginning, we list all predicates, i.e., concepts and
individuals from the input FOLE formulas. They will be used for the
search in WordNet.  In the current implementation we consider as {\em
  search predicates} all nouns, verbs, and named entities, together
with their sense information which is specified for each predicate by
the last number in the predicate name, e.g., sense $2$ in {\tt
  work\_n\_2}. In WordNet, the senses are generally ordered from most
to least frequently used, with the most common sense numbered {\tt
  1}. Frequency of use is determined by the number of times a sense
was tagged in the various semantic concordance texts used for
WordNet~\cite{fellbaum1998wordnet}.  Senses that were not semantically
tagged follow the ordered senses. For our small RTE problem we can
select as search predicates, e.g., {\tt work\_n\_2}, {\tt read\_v\_1},
or {\tt leibniz\_per\_1}. It is important for the integration that the
sense information computed during the semantic analysis matches
exactly the senses used by external knowledge sources. This ensures
that the semantic consistency of background knowledge is preserved
across the semantic and logical analysis. However, this seems to be an
extremely difficult task, which does not seem to be solved fully
automatically yet by any current word sense disambiguation
technique. Since in WordNet but also in
ERG~\cite{flickinger2000building} the senses are ordered by their
frequencies, we take for semantic representations generated during the
semantic analysis the most frequent concepts from ERG.

Having identified the search predicates, we try to find them in
WordNet and, by employing both the hyperonymy/hyponymy and synonymy
relations, we obtain a {\em knowledge graph} $G_{W}$. A small fragment
of such a knowledge graph for text $T$ of our example is given in
Figure~\ref{fig:wordnetdag}. In general, $G_{W}$ is a DAG with leaves
represented by the search predicates, whereas its inner nodes and the
root are given by concepts coming from WordNet. The directed edges in
$G_{W}$ correspond to the hyponym relations, e.g., in
Figure~\ref{fig:wordnetdag} named entity {\tt leipzig} is a hyponym of
concept {\tt city}. Note that in the opposite direction they describe
the hyperonym relations, e.g., concept {\tt city} is a hyperonym of
named entity {\tt leipzig}. Each synonymy relation is represented by a
{\em complex node} composed of synonymous concepts $C_1, ...,C_n$
induced by the relation (i.e., all concepts represented by a complex
node belong to the same synset in WordNet), e.g., the complex node
with concepts {\tt district} and {\tt territory} in
Figure~\ref{fig:wordnetdag}.
\begin{figure}
\begin{center}
\includegraphics[width=12cm]{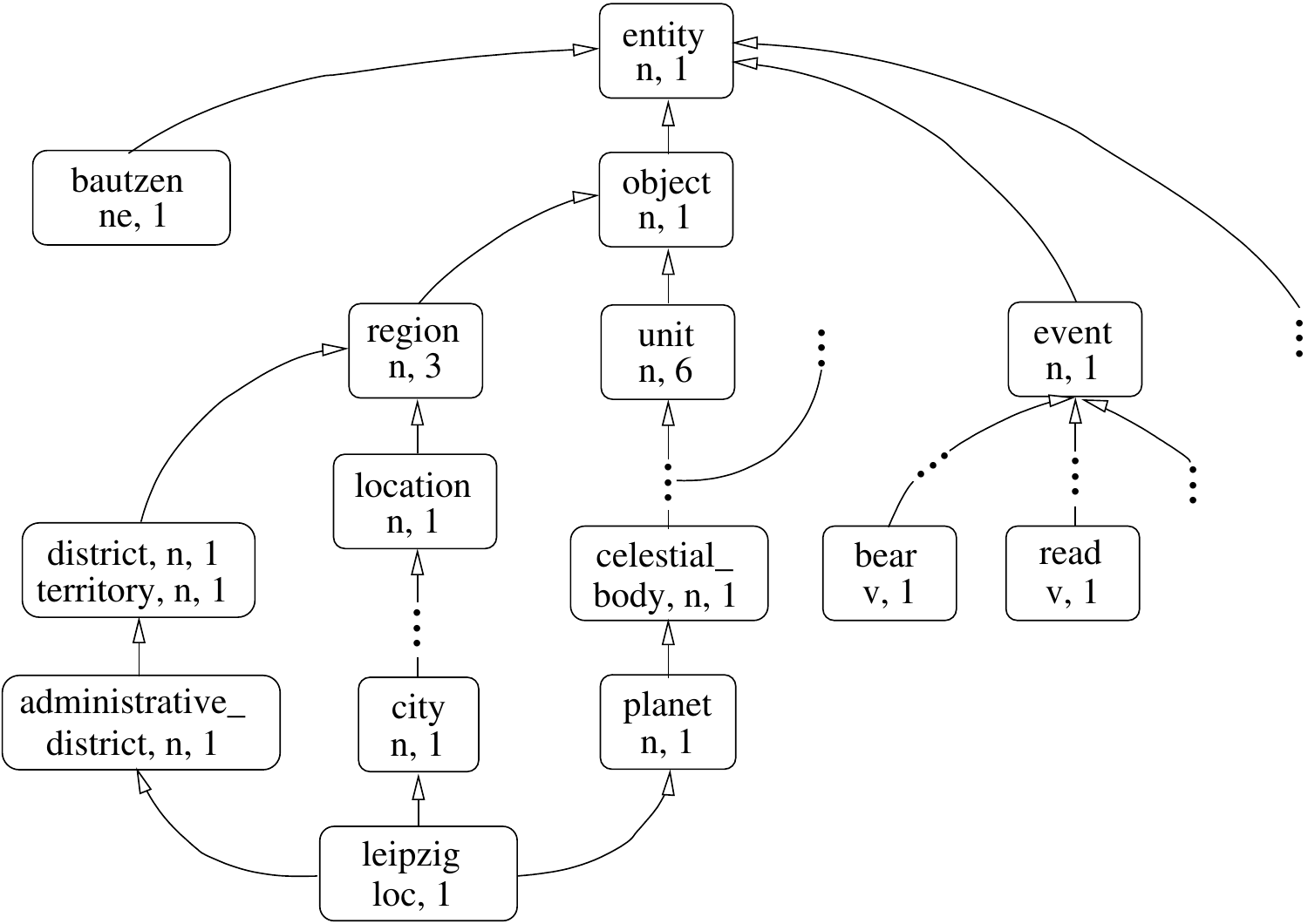}
\caption{Fragment of knowledge graph $G_{W}$ resulting from the search
  in WordNet}
\label{fig:wordnetdag}
\end{center}
\end{figure}

Furthermore, it can be seen in Figure~\ref{fig:wordnetdag} that the
leaf representing individual {\tt leipzig} has more than one direct
hyperonym, i.e., there are three hyponym relations for leaf {\tt
  leipzig} with concepts {\tt administrative\_district}, {\tt city},
and {\tt planet}. As already indicated in Section~\ref{sec:framework},
this property of graph $G_{W}$ may cause inconsistencies when the
background knowledge axioms are later generated from it and integrated
into the input FOLE formulas. We address this problem more carefully
now. We begin with the explanation how the background knowledge axioms
are generated. The method we use for it is an extension of the
heuristic presented in~\citeA{curran2007lingu} and can be defined
formally as follows:
\begin{definition}
\label{def:axioms}
There are three basic types of background knowledge axioms: {\tt
  IS-A}, {\tt IS-NOT-A}, and {\tt IS-EQ}. They can be generated from a
given knowledge graph $G$ by traversing its nodes and edges and
applying the following rules:
\begin{enumerate}
\item Let $U$ and $V$ be two different (complex) nodes from $G$, and
  $C_i$ and $C_j$ two arbitrary concepts or individuals represented by
  $U$ and $V$, respectively. If $C_i$ is a direct hyponym of $C_j$
  (i.e., there is an edge from $U$ to $V$ in $G$), then generate an
  {\tt IS-A} axiom $\forall x (C_i(x) \rightarrow C_j(x))$.
\item Let $V$ be a (complex) node from $G$ and $U=\{U_{1},...,U_{n}\}$
  the set of all children of $V$ in $G$. All concepts and individuals
  represented by (complex) nodes from $U$ are direct hyponyms of the
  concepts or individuals represented by $V$. The sets of concepts and
  individuals represented by nodes from $U$ are pairwise disjoint. For
  every pair $(i,j)$ such that $i=1,..,n-1$ and $j=i,...,n$ generate
  an {\tt IS-NOT-A} axiom $\forall x (C_i(x) \rightarrow \neg C_j(x))$
  where $C_i$ and $C_j$ are two arbitrarily chosen concepts or
  individuals represented by $U_{i}$ and $U_{j}$, respectively.
\item Let $U$ be some complex node from $G$ and $C=\{C_1,..,C_n\}$ a
  set of synonymous concepts represented by $U$. For every pair
  $(i,j)$ such that $i=1,..,n-1$ and $j=i,...,n$ generate an {\tt
    IS-EQ} axiom $\forall x (C_i(x) \leftrightarrow C_j(x))$.
\end{enumerate}
\end{definition}
Notice that since all concepts or individuals represented by a given
complex node are synonymous, Rule 1 and Rule 2 from
Definition~\ref{def:axioms} need to be applied only to one arbitrarily
chosen concept or individual represented by that node. By applying the
rules from Definition~\ref{def:axioms} to graph $G_{W}$ from
Figure~\ref{fig:wordnetdag} the following axioms can be generated (not
a complete list here):
\begin{eqnarray}
\mbox{{\tt IS-A}:}  & & \forall x (object\_n\_1(x) \rightarrow
entity\_n\_1(x)) \nonumber \\
\mbox{{\tt IS-NOT-A}:} & & \forall x (region\_n\_3(x) \rightarrow \neg
unit\_n\_6(x)) \nonumber \\
\mbox{{\tt IS-EQ}:} & & \forall x (district\_n\_1(x) \leftrightarrow 
territory\_n\_1(x)) \nonumber
\end{eqnarray}

Furthermore, observe that the set of all background knowledge axioms
$A_{K}$ generated for knowledge graph $G_{W}$ according to
Definition~\ref{def:axioms} is a finite set of first-order sentences
(i.e., formulas of first-order logic without free variables)
restricted to unary predicate symbols and no function symbols. This
monadic fragment of the first-order logic is know to be decidable for
logical validity~\cite{boolos1974computability}. However, in order to
show the consistency (i.e., the absence of contradictions) of $A_{K}$
generated for an arbitrary knowledge graph $G_{W}$, we need to show
that $A_{K}$ is {\em satisfiable}. We conjecture here that every
$A_{K}$ is satisfiable in some finite model. To prove this, one need
to give, for instance, some method which describes formally the
construction of a finite model for every set of axioms $A_{K}$
generated for an arbitrary knowledge graph $G_{W}$ according to
Definition~\ref{def:axioms}. Since the predicate calculus of first
order is complete~\cite{goedel1930vollstaendigkeit}, one can also
proceed in a purely syntactical way by showing that there is no
formula $f$ such that both $f$ and its negation are provable from
axioms $A_{K}$ under its associated deductive system. In the further
research we examine our conjecture more carefully, i.e., we will try
either to prove it or to deliver some counterexample.

\begin{proposition}
\label{t:integration}
Let $F$ be a set of FOLE formulas representing semantically an RTE
problem $P$, and $A_K$ a set of background knowledge axioms computed
for $P$ according to Definition~\ref{def:axioms}. Furthermore, let $f$
be a formula like $\exists x (C_{k}(x)\wedge...)$ from $F$ and
$A=\{A_{1},A_{2}, A_{3}\}=\{\forall x(C_i(x) \rightarrow \neg C_j(x)),
\forall x(C_k(x) \rightarrow C_i(x)), \forall x(C_k(x)\rightarrow
C_j(x))\}$ a set of one {\tt IS-NOT-A} and two {\tt IS-A} axioms,
respectively. If $A \subseteq A_K$, then $F \cup A_K$ is inconsistent.
\end{proposition}
\begin{proof}
  To show the inconsistency of $F \cup A_K$, we need to prove its
  unsatisfiability. Note that the three axioms from $A$ reflect the
  situation depicted in Figure~\ref{fig:proof}. To give a proof, we
  show first that $\{f\} \cup A$ is unsatisfiable. To this end we
  transform $\{f\} \cup A$ into an equivalent conjunctive normal
  form. The resulting set of clauses $\{\{f\}, \{A_{1}\},\{A_{2}\},
  \{A_{3}\}\}$ is unsatisfiable if and only if there exists a
  derivation of the empty clause using alone the resolution rule. It
  is clear that the empty clause can be derived, showing that $\{f\}
  \cup A$ is unsatisfiable and since $\{f\} \cup A \subseteq F \cup
  F_K$, the claim follows.
\end{proof}
\begin{figure}
\begin{center}
\includegraphics[width=4.5cm]{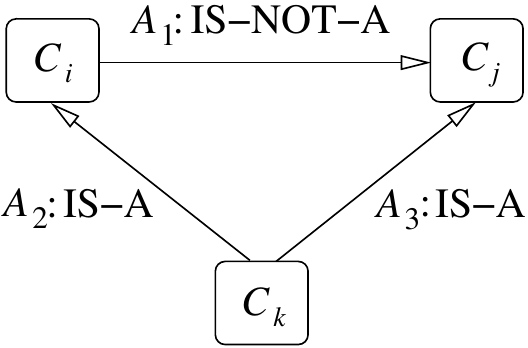}
\caption{Three background knowledge axioms leading into inconsistency}
\label{fig:proof}
\end{center}
\end{figure}

Thus, according to Proposition~\ref{t:integration}, we cannot in
general integrate all background knowledge axioms $A_{K}$ generated
from knowledge graph $G_{W}$ by the rules given in
Definition~\ref{def:axioms} into the RTE problem when its original
logical property (i.e., consistency or inconsistency) has to be
preserved. To deal with that problem, we propose two strategies:
\begin{enumerate}
\item Only Rule 1 and Rule 3 from Definition~\ref{def:axioms} are used
  for the generation of knowledge axioms from knowledge graph
  $G_{W}$.
\item Some edges from knowledge graph $G_{W}$ are removed, so that
  afterwards each concept or individual from $G_{W}$ is hyponym of
  concept(s) of at most one (complex) node, i.e., every child node in
  $G_{W}$ can have only one father node. Thus, after the deletion of
  edges is done, $G_{W}$ becomes a directed tree and all rules from
  Definition~\ref{def:axioms} are applied to it.
\end{enumerate}
Both strategies can cause some loss in effectiveness of the entire RTE
inference process. By using the first strategy, no {\tt IS-NOT-A}
axioms are generated and the situation described in
Proposition~\ref{t:integration} does not hold. However, the generated
background knowledge is not as precise as before (there are no
uniqueness constraints for concepts). The elimination of conflicting
edges from $G_{W}$ by the second strategy results in loss of
knowledge, too. For instance, in Figure~\ref{fig:proof} either the edge
representing axiom $A_{2}$ or the edge representing axiom $A_{3}$ will
be removed. To overcome this restriction and make use of all knowledge
from $G_{W}$, we could integrate all available hyponym relations into
the RTE problem separately, one after the other. Unfortunately, this
would result in many parallel entailment problems (one for each
reading), which we must solve and evaluate separately. Furthermore, we
observed that for now it is difficult to automate the task for
selecting edges for removal from $G_{W}$.

In our implementation we follow the second strategy by which we
transform knowledge graph $G_{W}$ into knowledge tree $T_{K}$ with
root node {\tt entity}, the most general concept in
WordNet~\cite{fellbaum1998wordnet}.  Currently, the edges for removal
can be selected either manually by the analyst from the list of
proposals made by the application, or automatically by leaving only
concepts with the most frequent senses.  Here, we use a variant of
Lesk's WSD algorithm~\cite{banerjee2002adapted}.
Figure~\ref{fig:wordnet} shows a fragment of tree $T_{K}$ for our
example.
\begin{figure}[h]
\begin{center}
\includegraphics[width=10cm]{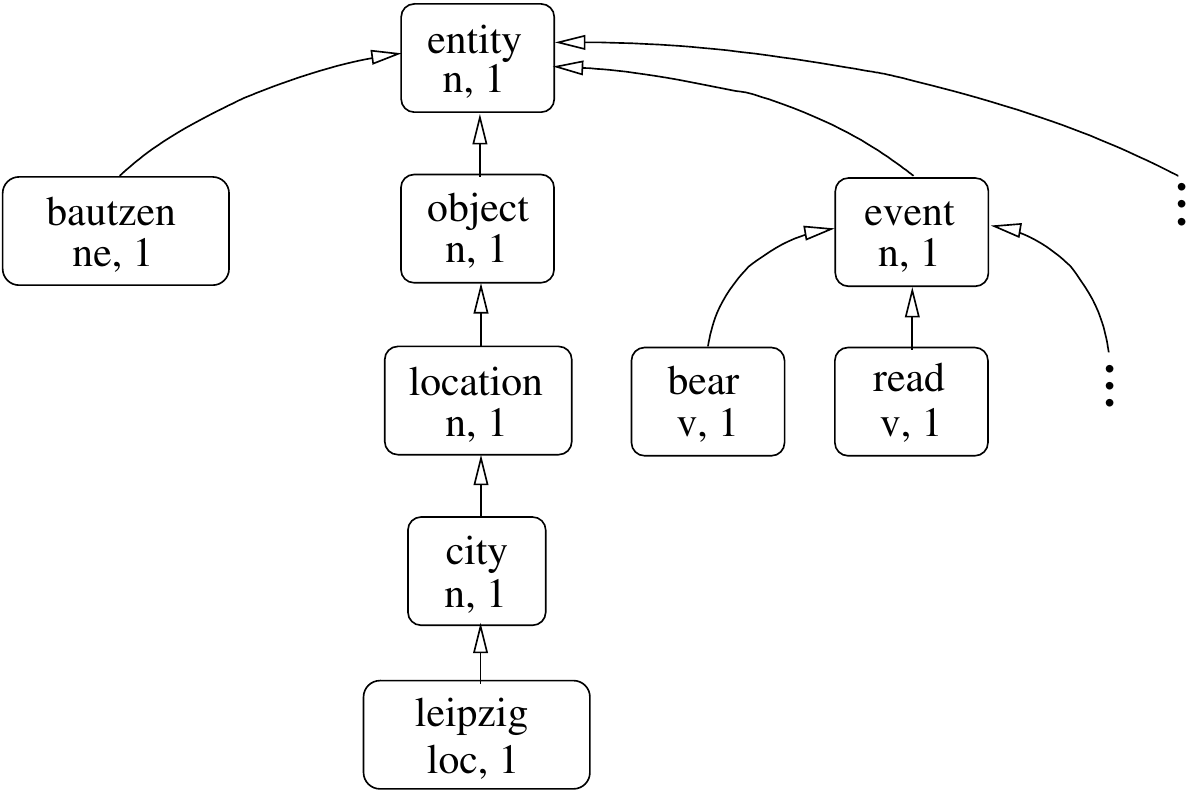}
\caption{Fragment of knowledge tree $T_{K}$ after optimization}
\label{fig:wordnet}
\end{center}
\end{figure}
The construction of the tree was optimized so that only those concepts
from $G_{W}$ appear in $T_{K}$ which are directly relevant for the
inference problem, e.g., only search predicates can serve as leaves in
$T_{K}$, or every non-branching node between two other nodes is
removed. Hence, all knowledge which will not add any inferential power
to the process is removed from $T_{K}$.

One can see in Figure~\ref{fig:wordnet} that not all search predicates
were recognized precisely enough during the first phase. More
specifically, the named entity {\tt bautzen} was not classified as a
town as we would expect it. Since a suitable individual was not found
in WordNet, the named entity {\tt bautzen\_ne\_1} was assigned
directly to the root of tree $T_{K}$. It is clear that without having
more information about {\tt bautzen}, we {\em cannot} prove the
entailment.

\subsubsection{Phase II: Integration of YAGO}
\label{sssec:secondphase}
In this phase we consult YAGO about search predicates that were not
recognized in the first phase. We formulate for each such predicate an
appropriate query and send it to the query processor (see
Figure~\ref{fig:semantics}). To this end, we use relation {\tt type},
one of the build-in ontological relations of YAGO. For our small RTE
problem, we ask YAGO with a query {\tt bautzen type ?} of what type
(or in YAGO nomenclature: of what class) the named entity {\tt
  bautzen} is. If it succeeds, it returns knowledge graph $G_{Y}$ with
WordNet concepts which classify the named entity from the
query. Figure~\ref{fig:dag} depicts graph $G_{Y}$ for our example. We
can see that {\tt bautzen} was now classified more precisely, among
other things, as a town.

In general, each graph $G_{Y}$ is a DAG composed of partially
overlapping paths leading (with respect to the hyperonymy relation)
from some root node (i.e., the most general concept in $G_{Y}$, e.g.,
node {\tt object} in Figure~\ref{fig:dag}) to the leaf representing
the search predicate (e.g., the complex node {\tt bautzen} in
Figure~\ref{fig:dag}). Observe that there is one and only one leaf
node in every graph $G_{Y}$. Since the result of every YAGO-query is
in general represented by a DAG, we cannot integrate it completely
into the knowledge tree $T_{K}$ without violating the original logical
properties of the input problem (see discussion above). According to
the leaf of $G_{Y}$ in Figure~\ref{fig:dag}, the named entity {\tt
  bautzen} can also be classified as an asteroid or an administrative
district.
\begin{figure}
\begin{center}
\includegraphics[width=15.2cm]{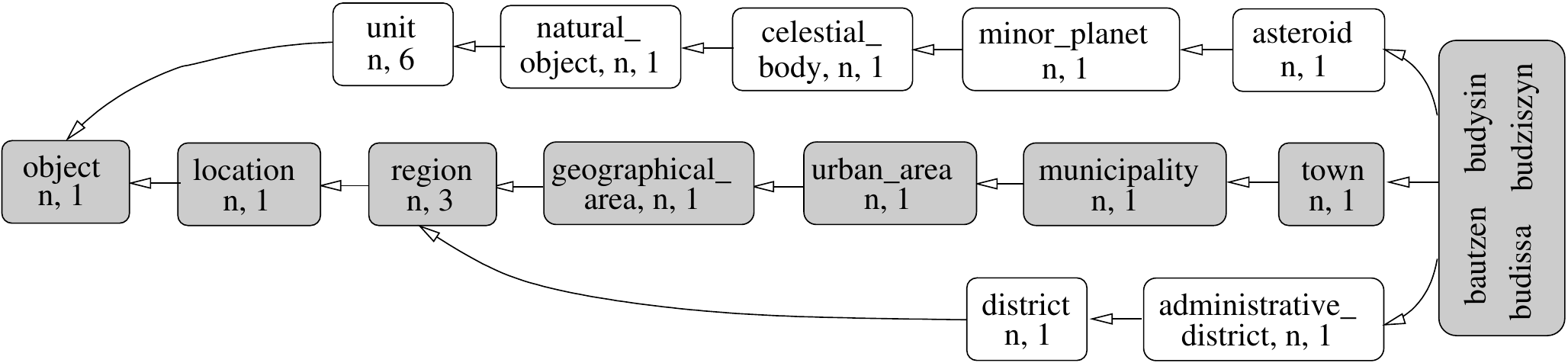}
\caption{Knowledge graph $G_{Y}$ with results of two queries to YAGO}
\label{fig:dag}
\end{center}
\end{figure}

In order to preserve the correctness of results, we select for the
integration into knowledge tree $T_{K}$ only those concepts,
individuals, and relations from $G_{Y}$ which lay on the longest path
from the most general concept in $G_{Y}$ to one of the direct
hyperonyms of the leaf, and which has the most common nodes with tree
$T_{K}$ from the first phase. In Figure~\ref{fig:dag} the concepts and
individuals on the gray shaded path were chosen by our heuristic for
the integration into $T_{K}$.  After the path has been selected, it is
optimized and integrated into $T_{K}$. Figure~\ref{fig:integration}
depicts the knowledge tree $T_{K}$ after the gray shaded path from
Figure~\ref{fig:dag} was integrated into it.
\begin{figure}
\begin{center}
\includegraphics[width=10.7cm]{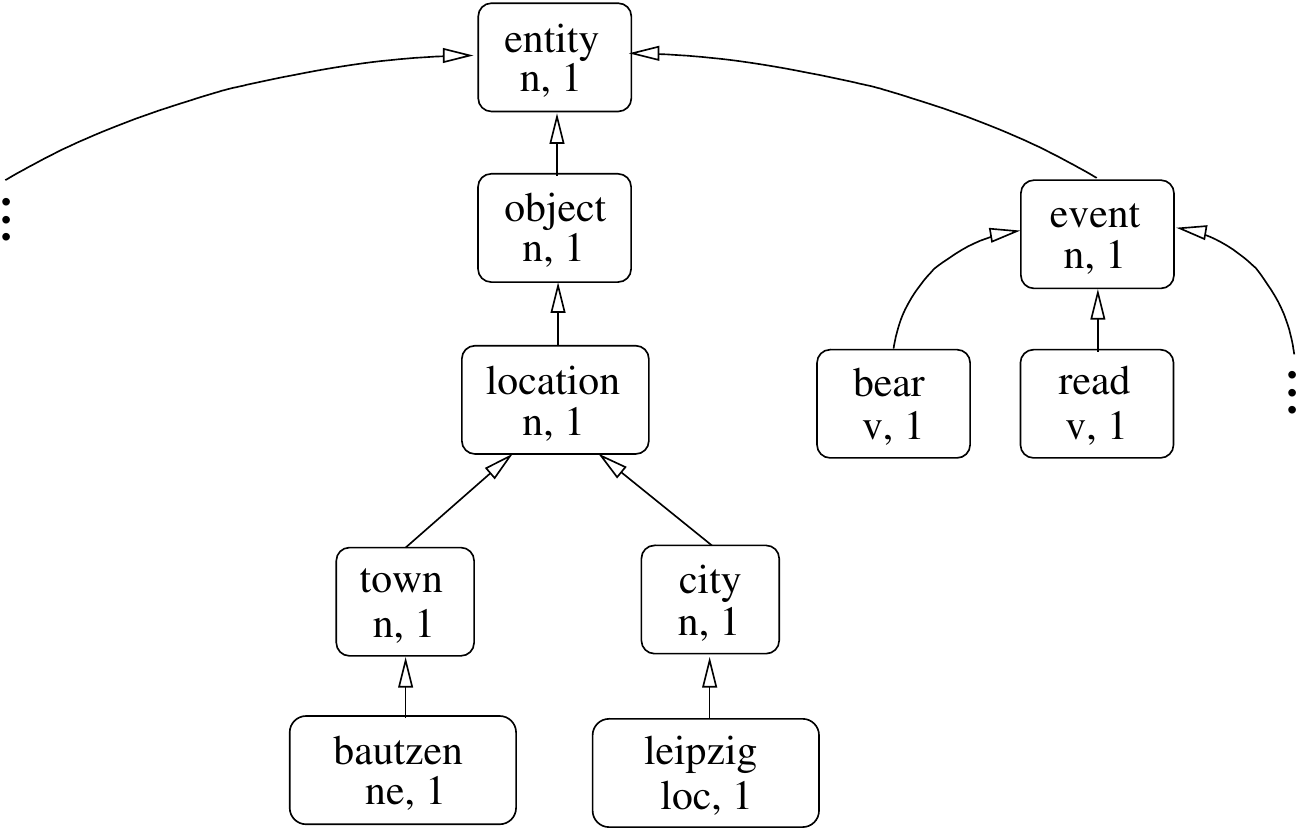}
\caption{Fragment of knowledge tree $T_{K}$ after integration of
  results from YAGO}
\label{fig:integration}
\end{center}
\end{figure}

The selection of a relevant path could also be done manually by some
analyst with sufficient knowledge about the problem. A fully automatic
selection turns out to be a much more difficult task. Another strategy
is to integrate all available paths separately, one after the
other. Here, however, this would result in three parallel entailment
problems, which we must solve and evaluate separately.

Observe finally that the integration of selected parts of the
knowledge graph $G_{Y}$ into tree $T_{K}$ is performed sequentially
for each search predicate which was not classified in the first phase
(note that each search generates its own knowledge graph $G_{Y}$).

Additionally to the first query to YAGO, we can also formulate a
second one like {\tt bautzen isCalled ?}, in which we ask what are the
names of the named entity in other languages. In Figure~\ref{fig:dag}
we can see four different names for this entity. This complementary
information can be combined afterwards into the FOLE formulas of the
RTE problem as new predicates, e.g.,
$$
...\exists x ((bautzen(x) \vee \mbox{\bf \em budysin}(x) \vee \mbox{\bf
  \em budissa}(x) \vee \mbox{\bf \em budziszyn}(x))\wedge ...)...
$$

\subsubsection{Phase III: Generation of Background Knowledge Axioms}
\label{sssec:thirdphase}
After the second phase of the integration procedure is finished and
the final knowledge tree $T_{K}$ has been computed, the background
knowledge axioms are generated from $T_{K}$ according to
Definition~\ref{def:axioms}. The resulting axioms are added into the
FOLE formulas of the input RTE problem. Such an extended input problem
is passed over to the inference process (see
Figure~\ref{fig:semantics}) and solved correspondingly. For the
knowledge tree given in Figure~\ref{fig:integration}, the following
axioms (here not a complete list) can be generated.

\begin{verbatim}
all(X, imp(city_n_1(X), location_n_1(X))).
all(X, imp(event_n_1(X), not(object_n_1(X)))).
\end{verbatim}

The axioms are in FOLE format and are rather self-explanatory and can
be interpreted as follows:
\begin{align} 
& \forall x (city\_n\_1(x) \rightarrow location\_n\_1(x)) \nonumber \\
& \forall x (event\_n\_1(x) \rightarrow \neg object\_n\_1(x)) \nonumber 
\end{align}

\subsection{Presupposition Resolution}
\label{ssec:presuppositions}
Many words and phrases trigger presuppositions which have clearly
semantic content important for the inference process. We try to
represent some of them explicitly. Our trigger-based mechanism uses
noun phrases as triggers, but it can be extended to verb phrases,
particles, etc. After a presupposition is triggered, the mechanism
resolves it, and integrates it as a new FOLE axiom into the RTE
problem. The automatic axiom generation is based on
$\lambda$-conversion and employs {\em abstract axioms} and a set with
possible {\em axiom arguments}. The axioms and their arguments are
still part of an experimental knowledge source (see Presuppositional
Knowledge in Figure~\ref{fig:semantics}). Here is an example for an
abstract axiom which allows for a translation from a noun phrase into
an intransitive verb phrase:
\begin{align}\label{eq:aaxiom}
  & \lambda P [\lambda R [\lambda S[ \nonumber \\
  & \mbox{\hspace{.3cm}}\forall x_1(\forall x_2 (P@x_1\wedge
  R@x_2\wedge
  nn\_r\_1(x_1,x_2) \\
  & \mbox{\hspace{.3cm}} \rightarrow \exists x_3(R@x_3\wedge \exists
  x_4 (S@x_4\wedge event\_n\_1(x_4) \wedge agent\_r\_1(x_4,x_3)))))
  \nonumber \\
  & ]]]. \nonumber
\end{align}

If text $T$ (expressed with FOLE formulas) contains a noun phrase
being a key for some entry in the set of possible axiom arguments,
then the arguments pointed by that key are applied to their abstract
axiom, and a new background axiom is generated. For a complex noun
phrase {\em price explosion} with its semantic representation
$price\_explosion\_n\_1$ the following arguments can be considered:
\begin{eqnarray}
  & \lambda x [explosion\_n\_1(x)]  \nonumber \\ 
  & \lambda x [price\_n\_1(x)] \nonumber \\ 
  & \lambda x [explode\_v\_1(x)] \nonumber
\end{eqnarray}
which after being applied to the abstract axiom~(\ref{eq:aaxiom})
produce the following background knowledge axiom:
\begin{align}\label{eq:axiom}
  & \forall x_1(\forall x_2 (explosion\_n\_1(x_1)\wedge
  price\_n\_1(x_2)
  \wedge nn\_r\_1(x_1,x_2) \nonumber \\
  & \rightarrow \exists x_3(price\_n\_1(x_3)\wedge \exists
  x_4(explode\_v\_1(x_4)\wedge event\_n\_1(x_4) \wedge 
  agent\_r\_1(x_4,x_3))))).
\end{align}
The presupposition axioms having complexity similar
to~(\ref{eq:axiom}) are first combined with the existing background
knowledge axioms and finally integrated as background knowledge into
the input RTE problem.

\section{Conclusion and Future Work}
\label{sec:conclusions}
In this paper a new adaptable, linguistically motivated system for RTE
was presented. Its deep-shallow semantic analysis, employing a
broad-coverage HPSG grammar ERG, was combined with a logical inference
process supported by an extended usage of external background
knowledge. The architecture of our system and the correctness of our
three-phase integration procedure were discussed in detail and the
functionality of the system was explained with several examples.

The system was successfully implemented and evaluated in terms of
success rate and efficiency. For now, it is still impossible to
measure its semantic accuracy as there is no corpus with gold standard
representations which would make comparison possible. Measuring
semantic adequacy could be done systematically by running the system
on controlled inference tasks for selected semantic phenomena.

Nevertheless, for our tests we used the RTE problems from the
development sets of the past RTE
Challenge~\cite{giampiccolo2007thirdpascal}. Our system with
successfully integrated background knowledge was able to solve
correctly about 67 percent of the RTE problems. This is comparable or
slightly better than the vast majority of other approaches from that
RTE Challenge which are based on some deep approach and combined with
logical inference. Unfortunately, it is still not so good as the
result of 72 percent achieved by~\citeA{tatu2006logicbased}. This can
be explained, among other things, by a more extensive and fine grained
usage of specific semantic phenomena, e.g., a sophisticated analysis
of named entities, in particular person names, distinguishing first
names from last names.

It is interesting to look at the inconsistent cases of the inference
process which were produced during the evaluation. They were caused by
errors in presupposition and anaphora resolution, incorrect syntactic
derivations, and inadequate semantic representations. They give good
indications for further improvements. Here, particularly the word
sense disambiguation problem will play a decisive role for matching
the set of senses of the semantic analyzers with multiple, and likely
different, sets of senses from the different knowledge resources. Once
tackled more precisely, it should decisively improve the success rate
of the system.  Moreover, the system presented here should be extended
with methods for word sense disambiguation, paraphrase detection, and
a better anaphora resolution within a discourse. We are considering
also the enhancing of the logical inference module with statistical
inference techniques in order to improve its performance and
recall. Since the strength but in some extent also the weakness of our
system lies in the difficulties regarding the computation of a
(nearly) full semantic representation of the input problem (see,
e.g.,~\citeA{burchardt2007semantic} for a good discussion), it might
to be recommended to integrate some models of natural language
inference which identifies valid inferences by their lexical and
syntactic features, without full semantic interpretation, like, e.g.,
the one proposed by~\citeA{maccartney2009extended}.

Furthermore, we intend to develop for the inference process some
temporal calculus supported by the temporal information from
YAGO. Here, the event calculus of~\citeA{shanahan1999event} can be
considered as a good starting point. Finally, it would be interesting
to extend the semantic analysis of our system, so that RTE problem
instances in languages other than English could be supported.


\vskip 0.2in
\bibliography{mIE-Literatur}
\bibliographystyle{theapa}

\end{document}